\documentclass[conference]{IEEEtran}
\IEEEoverridecommandlockouts
\usepackage{cite}
\usepackage{amsmath,amssymb,amsfonts}
\usepackage{amsthm}
\usepackage{algorithmic}
\usepackage{graphicx}
\usepackage{textcomp}
\usepackage{xcolor}
\usepackage{mathtools}
\usepackage{threeparttable}
\usepackage{subfig}
\usepackage{float}
\usepackage{overpic}
\usepackage{pifont}
\usepackage{color}
\usepackage{algorithm}

\def\BibTeX{{\rm B\kern-.05em{\sc i\kern-.025em b}\kern-.08em
    T\kern-.1667em\lower.7ex\hbox{E}\kern-.125emX}}
\begin{document}

\title{MOB-FL: Mobility-Aware Federated Learning for Intelligent Connected Vehicles}

\author{\IEEEauthorblockN{
Bowen Xie\IEEEauthorrefmark{1}, 
Yuxuan Sun\IEEEauthorrefmark{2}, 
Sheng Zhou\IEEEauthorrefmark{1}, 
Zhisheng Niu\IEEEauthorrefmark{1}, 
Yang Xu\IEEEauthorrefmark{3}, 
Jingran Chen\IEEEauthorrefmark{3},
Deniz G\"und\"uz\IEEEauthorrefmark{4}}\\
 \IEEEauthorblockA{\IEEEauthorrefmark{1}Beijing National Research Center for Information Science and Technology\\
  Department of Electronic Engineering, Tsinghua University, Beijing 100084, China\\
  \IEEEauthorrefmark{2}School of Electronic and Information Engineering, Beijing Jiaotong University, Beijing 100044, China\\
  \IEEEauthorrefmark{3}Department of Standard and Research, OPPO, Beijing, China\\
  \IEEEauthorrefmark{4}Department of Electrical and Electronic Engineering, Imperial College London, London SW7 2BT, UK\\
  Email: xbw22@mails.tsinghua.edu.cn, yxsun@bjtu.edu.cn,\\
   \{sheng.zhou, niuzhs\}@tsinghua.edu.cn,
   \{xuyang, chenjingran\}@oppo.com,
   d.gunduz@imperial.ac.uk }

}

\maketitle

\begin{abstract}
Federated learning (FL) is a promising approach to enable the future Internet of vehicles consisting of intelligent connected vehicles (ICVs) with powerful sensing, computing and communication capabilities. We consider a base station (BS) coordinating nearby ICVs to train a neural network in a collaborative yet distributed manner, in order to limit data traffic and privacy leakage.
However, due to the mobility of vehicles, the connections between the BS and ICVs are short-lived, which affects the resource utilization of ICVs, and thus, the convergence speed of the training process. 
In this paper, we propose an accelerated FL-ICV framework, by optimizing the duration of each training round and the number of local iterations, for better convergence performance of FL. 
We propose a mobility-aware optimization algorithm called MOB-FL, which aims at maximizing the resource utilization of ICVs under short-lived wireless connections, so as to increase the convergence speed. Simulation results based on the beam selection and the trajectory prediction tasks verify the effectiveness of the proposed solution.

\end{abstract}


\section{Introduction}

Intelligent connected vehicles (ICVs) play an important role in the future Internet of vehicles (IoV), where vehicles can communicate via the vehicle-to-everything (V2X) technologies, such as dedicated short range communication (DSRC) and cellular V2X (C-V2X) \cite{v2x2}.
However, due to the mobility of vehicles and the complex traffic flows, many typical tasks for ICVs, such as driving trajectory prediction, traffic flow prediction, and smart V2X communication, are facing highly dynamic environments. Data-driven machine learning (ML) solutions are emerging as a promising approach to tackle these complex tasks.

ICVs are equipped with multiple sensors like cameras, GPS, and LiDAR, and generate abundant real-time data to be used by ML algorithms. However, centralized training is not an option in practice, both due to the high communication cost of offloading such amounts of distributed data to a cloud server, and the associated privacy concerns. Federated learning (FL) is a promising solution, where a central server coordinates end devices to collaboratively train a neural network (NN) model in a distributed fashion \cite{mcmahan2017}. 
Recent studies have shown the promise of FL for vehicular networks \cite{meet, du2020}.
However, there are still many challenges for the FL with ICVs, such as the mobility of vehicles, limited computing and communication resources, and dynamic IoV environments \cite{du2020,chen2021distributed}. 

Existing studies about FL over wireless networks or vehicular networks have been focusing on the optimization strategies of resource allocation \cite{yang2020energy}, device scheduling \cite{yang2019scheduling, amiri2021convergence, sun2021dynamic}, model aggregation \cite{ye2020}, and their joint optimization \cite{shi2021, wang2021}. Their main objective is to improve the convergence rate of FL while saving energy and satisfying latency constraints. 
Specifically, by jointly optimizing the device scheduling and spectrum resource allocation, the work \cite{shi2021} maximizes the model training accuracy given the total training time.
Since the computing capability and data quality of ICVs can affect the training efficiency, they are taken into account for the design of model aggregation approaches in \cite{ye2020} and \cite{wang2021}. 
References \cite{nori2021} and \cite{prakash2021} balance the trade-off between the computing and communication latency to reduce the convergence time. 
A multi-layer FL framework and the related heterogeneous model aggregation strategies are also proposed in the IoV scenario in \cite{zhou2021}.

\begin{figure*}[htbp]
\centerline{\includegraphics[width=0.95\textwidth, trim=0 310 0 0,clip]{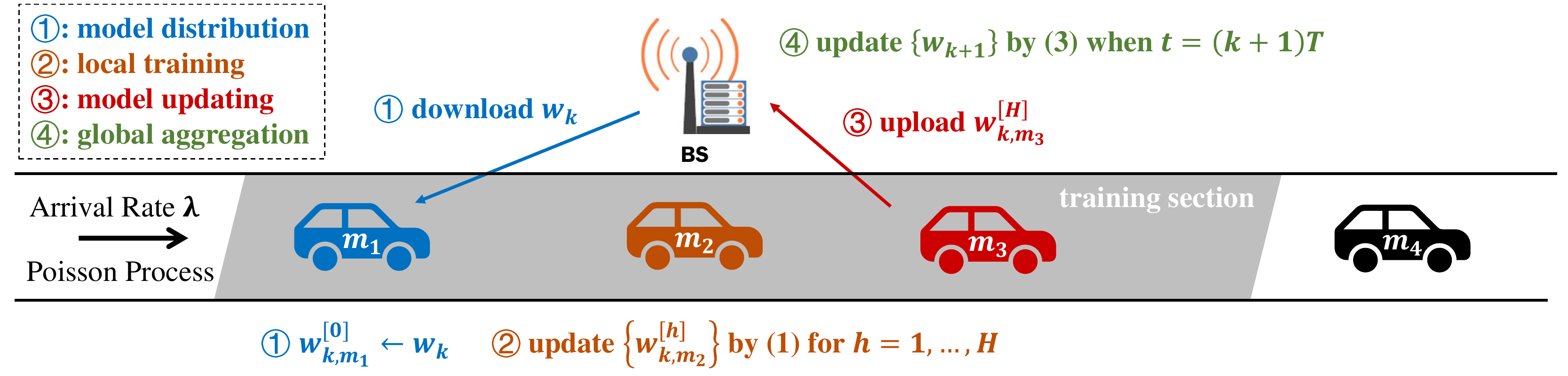}}
\caption{The FL-ICV framework.}
\label{fig:system model}
\end{figure*}

Nevertheless, these papers do not consider the \emph{dynamic scenario}, where the set of schedulable devices is time-varying during the FL process due to the mobility of vehicles. Moreover, ICVs may fail to upload their local models to the BS before they move away and lose connection. We notice that some papers like \cite{yu2020mobility, xiao2021vehicle} also consider the dynamic case. However, different from those papers, we focus on the optimization of \emph{two major hyper-parameters}, the \emph{round duration} and the \emph{local iteration number}, which affect the convergence speed of FL.

In this paper, we consider an FL-ICV framework, where a BS coordinates the FL process as the parameter server, and the passing ICVs participate in FL as end devices. Assuming the Poisson arrival of ICVs at the coverage area of the BS, we analyze the probability distribution of the number of ICVs that successfully upload their local models in each round, and prove that it also follows a Poisson distribution. Based on the analytical results, we formulate an optimization problem to maximize the convergence speed of FL, by determining the optimal round duration and the local iteration number. A mobility-aware optimization algorithm called MOB-FL is proposed to solve this problem. 
Two practical scenarios, namely \emph{millimeter-wave beam selection} and \emph{driving trajectory prediction}, are considered in the experiments, using Raymobtime \cite{raymobtime} and Argoverse \cite{argoverse} datasets, respectively.
Results of experiments show the effectiveness of the proposed optimization algorithm.

\section{System Model and Problem Formulation}{\label{sec:System Model and Problem Formulation}}

\subsection{The FL-ICV Framework}
As shown in Fig. \ref{fig:system model}, we consider an FL-ICV framework, where a BS orchestrates the training of a NN model $\boldsymbol{w}$, with the help of ICVs passing by. The BS acts as the parameter server, covering a road section of length $L$, called the \emph{training section}. The ICVs serve as the devices participating in the FL process when and only when driving through the training section. The speed of ICVs is denoted by $v$ and the arrival of ICVs at the training section follows a Poisson process with rate $\lambda$.

Similar to the traditional federated learning (TFL) framework \cite{mcmahan2017}, the BS in the FL-ICV framework is responsible for model dissemination and aggregation, while ICVs train the model with their local datasets using their on-board computing resources. We assume that local data samples are generated before the arrival of ICVs (e.g., past driving trajectories for motion forecasting). The goal of training is to minimize the global loss function $\mathcal{F}\left( \boldsymbol{w}, \mathcal{B}_{\rm{glob}} \right)$, i.e., the average loss of model $\boldsymbol{w}$ on the global dataset $\mathcal{B}_{\rm{glob}}$.

In the FL-ICV framework, the duration of each training round is denoted by $T$, and the number of local iterations in each round is denoted by $H$. We denote the set of ICVs which appear during the $k$-th round by $\mathcal{M}_k$, with cardinality $M_k$.
Different from TFL, where the set of available devices is fixed, the set of schedulable ICVs is time-varying, which means $\mathcal{M}_k$ and $M_k$ change over time. More precisely, each ICV can contribute to training only over a limited time duration, while passing by the BS.

There are four stages in each round of FL-ICV:
\subsubsection{Model distribution}
At the beginning of the $k$-th round, the BS distributes the current global model $\boldsymbol{w}_k$ to the ICVs within the training section. New ICVs may arrive at the training section during the rest of this round, and $\boldsymbol{w}_k$ is transmitted to all the ICVs upon their arrivals. We denote the communication delay caused by the model distribution from the BS to the ICV-$m$ in the $k$-th round by $\tau^{\rm{down}}_{k,m}$.

\subsubsection{Local training}
After receiving the global model $\boldsymbol{w}_k$, the ICV-$m$ performs local training for $H$ iterations using stochastic gradient descent (SGD):
\begin{equation}
    \label{eq:FL local update}
    \boldsymbol{w}_{k,m}^{[h]} = \boldsymbol{w}_{k,m}^{[h-1]} - \eta \nabla \mathcal{F}\left(\boldsymbol{w}_{k,m}^{[h-1]},\mathcal{B}_{k,m}^{[h]}\right), h = 1,\dots,H,
\end{equation}
where $\boldsymbol{w}_{k,m}^{[h]}$ is the local model of ICV-$m$ after the $h$-th local iteration in the $k$-th round, and $\boldsymbol{w}_{k,m}^{[0]} = \boldsymbol{w}_k$. Parameter $\eta$ is the learning rate, $\mathcal{B}_{k,m}^{[h]}$ is a batch of data sampled randomly from the local dataset $\mathcal{D}_m$ of ICV-$m$ for its $h$-th local iteration.
We denote the local training delay of ICV-$m$, from receiving $\boldsymbol{w}_k$ until obtaining $\boldsymbol{w}_{k,m}^{[H]}$, by $\tau^{\rm{cp}}_{k,m}$, which is influenced by the computing power and workload of ICV-$m$, the computational complexity of the training task, and the local iteration number $H$. The computing delay $\tau^{\rm{cp}}_{k,m}$ is modeled as a random variable following a shifted exponential distribution \cite{lee2018}:
\begin{equation}
    \mathbb{P}\left\{\tau_{k,m}^{\rm{cp}} \leq t \right\} = \left \{
    \begin{array}{ll}
    1-e^{-\frac{1}{\beta H}\left(t-\alpha H\right)}, & t \geq \alpha H, \\
    0, & \rm{otherwise}, \\
    \end{array}
    \right.
\label{eq:local train delay}
\end{equation}
where $\alpha$ is the minimum computing delay for one local iteration and $\beta$ is a parameter characterizing randomness. 

\subsubsection{Model uploading}
On completing local training, ICVs upload their updated local models to the BS, as long as they are still within the training section. Note that an ICV may fail to upload its local model if the ICV leaves the training section before the completion of its model uploading. The set of ICVs that successfully upload their models within the $k$-th round is denoted by $\mathcal{M}^{\rm{suc}}_k$ with cardinality $M^{\rm{suc}}_k$, where $\mathcal{M}_k^{\rm{suc}} \subseteq \mathcal{M}_k$ and $M^{\rm{suc}}_k \leq M_k$. The communication delay caused by the model uploading from ICV-$m$ to the BS in the $k$-th round is denoted by $\tau^{\rm{up}}_{k,m}$.

\subsubsection{Global aggregation}
At the end of the $k$-th round, the BS aggregates the received local models to update the global model:
\begin{equation}
    \label{eq:FL-ICV aggregation alpha}
    \boldsymbol{w}_{k+1} = \sum_{m \in \mathcal{M}^{\rm{suc}}_k} \frac{D_m}{\sum_{m^{\prime} \in \mathcal{M}^{\rm{suc}}_k} {D_{m^{\prime}}}} \boldsymbol{w}_{k,m}^{[H]},
\end{equation}
where $D_m$ is the cardinality of the dataset of ICV-$m$.

\subsection{Problem Formulation}
In this paper, the optimization objective is to maximize the convergence speed of the global model in the FL-ICV framework.
To measure the convergence speed, we can compare the final minimum loss $\mathcal{L}_{\rm{min}}(t)$ of the global model after a given period of training time $t$, which is defined as:
\begin{equation}
    \mathcal{L}_{\rm{min}}(t) \triangleq \min_{k\in\left\{0,1,...,{\left\lfloor \frac{t}{T} \right\rfloor} \right\}} \mathcal{F}\left(\boldsymbol{w}_k,\mathcal{D}_{\rm{val}}\right), 
    \label{eq:final minimum loss}
\end{equation}
where $\mathcal{D}_{\rm{val}}$ is the validation dataset. Then our objective is to minimize $\mathcal{L}_{\rm{min}}(t)$ by optimizing the round duration $T$ and the local iteration number $H$. 
However, it is difficult to directly solve the optimization problem based on \eqref{eq:final minimum loss}, as the convergence process of FL is very complex. To overcome this difficulty, we consider a heuristic problem in the following.

Existing papers such as \cite{li2019} show that the convergence rate of FL is proportional to the frequency of model updates. Since the model update of a round may be invalid if there is no ICV successfully uploading the model to the BS in this round, we introduce a new objective function $g(H,T)$ which represents the frequency of \textit{valid} model updates in the FL-ICV framework:
\begin{equation}
    g(H,T) \triangleq \frac{{H}}{{T}}\cdot {\mathbb{P}\left\{{ M^{\rm{suc}}_k > 0 | H,T }\right\}}.
    \label{eq:objective function}
\end{equation}
The design of \eqref{eq:objective function} is motivated by the following: 1) to accelerate the FL process, one solution is to carry out more frequent global aggregations, i.e., to shorten the round duration ${T}$; 2) to make the local training more efficient, more SGD iterations could be taken in each round, i.e., to increase the local iterations $H$; 3) to achieve a higher proportion of valid global aggregations, the probability term ${\mathbb{P}\left\{ M^{\rm{suc}}_k > 0 \right\}}$ should also be considered.

Then the optimization problem can be reformulated as:
\begin{align} 
    \text{$\mathcal{P}$0:} \quad \max_{H,T} \quad & g(H,T)\\
    \text{s.t.} \quad  & T \in \mathbb{R}^+, \, H \in \mathbb{N}^+.
\end{align}

Although $g(H,T)$ is designed from intuition, the simulation results in Section \ref{sec:Simulation Results} will show its strong correlation with $\mathcal{L}_{\rm{min}}(t)$.
Note that the design of $g(H,T)$ \emph{does not} rely on any assumptions like \cite{wang2019,shi2021} that the loss function is convex, $\rho$-Lipschitz and $\beta$-smooth.

\section{Performance Analysis and Optimization}
In this section, we first analyze the probability distribution of $M^{\rm{suc}}_k$, i.e., the number of ICVs which successfully upload models during the $k$-th round.
Based on the analysis, we obtain the detailed expression of $g(H,T)$ and the bounded area in which the optimal solution $(H^*,T^*)$ to \text{$\mathcal{P}$0} exists. Finally, we design an optimization algorithm called MOB-FL to solve \text{$\mathcal{P}$0}.

\subsection{The probability distribution of $M^{\rm{suc}}_k$}
There are several factors influencing the probability distribution of $M^{\rm{suc}}_k$ in the FL-ICV framework, such as the arrival rate $\lambda$ and velocity $v$ of ICVs, the training section length $L$, the round duration $T$, the communication delays $\tau^{\rm{down}}_{k,m}$ and $\tau^{\rm{up}}_{k,m}$, and the computing delay $\tau^{\rm{cp}}_{k,m}$ affected by the local iteration number $H$. 

For the simplicity of analysis, we consider $\tau^{\rm{down}}_{k,m}$ and $\tau^{\rm{up}}_{k,m}$ as constant values $\tau^{\rm{down}}$ and $\tau^{\rm{up}}$, respectively. The speed of each ICV is assumed to be a constant value $v$. So the duration each ICV stays within the training section is $T_0 = \frac{L}{v}$. 

For any ICV $m \in \mathcal{M}_k$, its arrival time $\zeta^A_m$ must be within the time interval $\left(kT-T_0, \left(k+1\right)T\right)$.
Let $\Phi^{[1]}_{k,m}$ be the instant when the ICV-$m$ successfully completes model uploading and $\Phi^{[2]}_{k,m}$ be the deadline for the ICV-$m$ to complete model uploading. 
It holds that $m \in \mathcal{M}^{\rm{suc}}_k$ if and only if:
\begin{equation}
    \Phi^{[1]}_{k,m} \leq \Phi^{[2]}_{k,m}.
    \label{eq:condition for m in suc}
\end{equation}
Since only after the ICV-$m$ arrives on the training section and the $k$-th round begins can the ICV-$m$ start to download the global model $\boldsymbol{w}_k$ from the BS, we have:
\begin{equation}
    \Phi^{[1]}_{k,m} = \max\left\{kT,\zeta^A_m\right\} + \tau^{\rm{down}}_{k,m} + \tau^{\rm{cp}}_{k,m} + \tau^{\rm{up}}_{k,m}.
    \label{eq:Phi1}
\end{equation}
Considering the ICV-$m$ leaves the training section at time $\zeta^A_m +T_0$ and the $k$-th round ends at time $(k+1)T$, the deadline $\Phi^{[2]}_{k,m}$ should be the minimum of the two:
\begin{equation}
    \Phi^{[2]}_{k,m} = \min \left\{\zeta^A_m +T_0,\left(k+1\right)T\right\}.
    \label{eq:Phi2}
\end{equation}
Note that $\zeta^A_m$ and $\tau^{\rm{cp}}_{k,m}$ are random variables. According to \eqref{eq:local train delay} and \eqref{eq:condition for m in suc}-\eqref{eq:Phi2}, we obtain a necessary condition for $m \in \mathcal{M}^{\rm{suc}}_k$:
\begin{equation}
    \Xi(H,T) \triangleq \min \left\{ T, T_0 \right\} - \mathcal{T}_{\rm{min}}(H) > 0,
    \label{eq:necessary condition for suc}
\end{equation}
where
\begin{equation}
    \mathcal{T}_{\rm{min}}(H) \triangleq \alpha H + \tau^{\rm{down}} + \tau^{\rm{up}}.
\end{equation}

Note that this condition is independent of the ICV's index $m$, arrival time $\zeta^A_m$, and the round index $k$. In other words, no ICVs can successfully upload their updated local models in a round if condition \eqref{eq:necessary condition for suc} is not satisfied.

\newtheorem{theorem}{Theorem}
\newtheorem{proposition}{Proposition}

\begin{theorem}
\label{theorem:1}
$M^{\rm{suc}}_k$ follows a Poisson distribution with parameter $\Lambda(H,T)$ given by \eqref{eq:theorem:Lambda}, if condition \eqref{eq:necessary condition for suc} holds.
\begin{multline}
        \Lambda(H,T) = \\
        2 \lambda\Xi(H,T) + \lambda\left(1-e^{-\frac{\Xi(H,T)}{\beta H}}\right)\left(\left|T-T_0\right|-2 \beta H\right).
    \label{eq:theorem:Lambda}
\end{multline}
\end{theorem}
\begin{proof}
    See Appendix \ref{appendix:theorem 1}.
\end{proof}
We have ${\mathbb{P}\left\{{ M^{\rm{suc}}_k > 0 | H,T }\right\}} = 1-e^{-\Lambda(H,T)}$ according to {Theorem \ref{theorem:1}}.

\subsection{Optimization Problem}
\begin{proposition}{\label{proposition:T up bound}}
When condition \eqref{eq:necessary condition for suc} holds, if $\frac{\partial g(H,T)}{\partial T} \geq 0$, then $T \leq \mathcal{T}_{\rm{max}}(H)$, where $\mathcal{T}_{\rm{max}}(H)$ is given by
\begin{equation}
    \mathcal{T}_{\rm{max}}(H) = 
        \begin{cases}
        T_0, &\mbox{if $\mathcal{C}_1(H) \geq 0$,}\\
        T_0 + \frac{1-12\lambda \mathcal{C}_1(H)}{4\lambda \mathcal{C}_0(H)}, &\mbox{if $\mathcal{C}_1(H) < 0$,}
        \end{cases}
    \label{eq:proposition:T up bound}
\end{equation}
where
\begin{equation}
    \mathcal{C}_0(H) = 1-e^{-\frac{T_0 - \mathcal{T}_{\rm{min}}(H)}{\beta H}},
    \label{eq:proposition:C_0}
\end{equation}
\begin{equation}
    \mathcal{C}_1(H) = 2
    \left(T_0 - \mathcal{T}_{\rm{min}}(H)\right) - \left(T_0 + 2\beta H\right)\mathcal{C}_0(H).
    \label{eq:proposition:C_1}
\end{equation}
\end{proposition}
\begin{proof}
    See Appendix \ref{appendix:proposition:T up bound}.
\end{proof}

\begin{proposition}{\label{proposition:unimodel}}
For a given $H$, $g(H,T)$ is a unimodal function for $T \in \left[\mathcal{T}_{\rm{min}}(H), \infty \right)$. 
\end{proposition}
\begin{proof}
    See Appendix \ref{appendix:proposition:unimodel}.
\end{proof}

According to {Proposition \ref{proposition:T up bound}}, there is no need to consider the $T > \mathcal{T}_{\rm{max}}(H)$  case since our goal is to maximize $g(H,T)$. Besides, we can obtain from condition \eqref{eq:necessary condition for suc} that the lower-bound of $T$ is $\mathcal{T}_{\rm{min}}(H)$, and the upper-bound of $H$ is $\alpha^{-1}\left( \min \{ T,T_0 \}-\tau^{\rm{down}}-\tau^{\rm{up}} \right)$. Then, \text{$\mathcal{P}$0} is transformed to \text{$\mathcal{P}$1} \emph{without loss of optimality}:
\begin{subequations}
\begin{alignat}{4}
    \text{$\mathcal{P}$1:} \quad \max_{H,T} \quad & \frac{H}{T}\left(1-e^{-\Lambda(H,T)}\right) \label{formulation:optimized function}\\
    \text{s.t.} \quad  & \mathcal{T}_{\rm{min}}(H) < T \leq \mathcal{T}_{\rm{max}}(H), \label{formulation:T constraint} \\
    & H \leq \left\lfloor \frac{\min \{ T,T_0 \}-\tau^{\rm{down}}-\tau^{\rm{up}}}{\alpha} \right\rfloor, \label{formulation:H constraint} \\
    & T \in \mathbb{R}^+, \, H \in \mathbb{N}^+. \label{formulation:range constraint} 
\end{alignat}
\end{subequations}

\addtolength{\topmargin}{0.05in}

\subsection{Optimization Algorithm}
To solve \text{$\mathcal{P}$1}, an intuitive idea is to first find the local optimal $T^{[H]}$ for every $H$ under constraint \eqref{formulation:H constraint} and then traverse them to find the global optimal solution $H^*$ and $T^*$ by comparing $\left\{g\left(H,T^{[H]}\right)\right\}_H$ values. It will be shown that one can find an approximate optimal solution following the proposed $H$-$T$ joint optimization algorithm called MOB-FL, which is summarized in {Algorithm \ref{alg:optimization algorithm}}. We emphasize that MOB-FL is a mobility-aware algorithm, since it considers the speed $v$ of ICVs, as well as the arrival rate $\lambda$ of ICVs at the training section.

For a fixed $H$, the range of $T \in \mathbb{R}^+$ is decided by \eqref{formulation:T constraint}. According to {Proposition \ref{proposition:unimodel}}, we can find the optimal value of $T$ that maximizes $g(H,T)$, denoted by $T^{[H]}$, by solving the equation $\partial_T g(H,T) = 0$. Since this equation has no closed-form solution, we use bisection method with threshold $\gamma$ to obtain its approximate optimal solution $T^{[H]}_\gamma$. 

\renewcommand{\algorithmicrequire}{\textbf{Input:}}  
\renewcommand{\algorithmicensure}{\textbf{Output:}} 
\begin{algorithm}
  \caption{MOB-FL algorithm}
  \label{alg:optimization algorithm}
  \small
  \begin{algorithmic}[1]
    \REQUIRE $L,\,v,\,\tau^{\rm{down}},\,\tau^{\rm{up}},\,\alpha,\,\beta,\,\lambda,\,\gamma$
    \ENSURE $H^*_\gamma$, $T^*_\gamma$
    \STATE Initialization: $T_0=\frac{L}{v}$,\quad $H_{\rm{max}} = \left\lfloor \frac{T_0-\tau^{\rm{down}}-\tau^{\rm{up}}}{\alpha} \right\rfloor$
    \FOR{$h = 1,\dots,H_{\rm{max}}$}
        \STATE $T_{\rm{min}} \leftarrow \mathcal{T}_{\rm{min}}(H)$
        \STATE $T_{\rm{max}} \leftarrow \mathcal{T}_{\rm{max}}(h)$, where $\mathcal{T}_{\rm{max}}(\cdot)$ is given by \eqref{eq:proposition:T up bound}
        \STATE $T \leftarrow \frac{1}{2}\left(T_{\rm{min}} + T_{\rm{max}}\right)$
        \WHILE{$ T_{\rm{max}} - T_{\rm{min}} > \gamma$}
            \STATE \textbf{if} $\partial_T g(h,T) > 0$ \textbf{then} $T_{\rm{min}} \leftarrow T$
            \STATE \textbf{else} $T_{\rm{max}} \leftarrow T$
            \STATE $T \leftarrow \frac{1}{2}\left(T_{\rm{min}} + T_{\rm{max}}\right)$
        \ENDWHILE
        \STATE $T^{[h]}_\gamma \leftarrow T$, $g^{[h]}_\gamma \leftarrow g(h,T^{[h]}_\gamma)$
        
    \ENDFOR

    \STATE $H^*_\gamma \leftarrow \arg \max_{h \in \{1,\dots,H_{\rm{max}}\}} g^{[h]}_\gamma$, $T^*_\gamma \leftarrow T^{[H^*_\gamma]}_\gamma$

  \end{algorithmic}
\end{algorithm}

We repeat this process of optimizing $T$ for every $H \in \{ 1,\dots,H_{\rm{max}} \}$, where $H_{\rm{max}} = \alpha^{-1}\cdot\left( T_0 - \tau^{\rm{down}} - \tau^{\rm{up}} \right)$, then we can obtain an approximate optimal solution $(H^*_\gamma,T^*_\gamma)$ of \text{$\mathcal{P}$1} by
\begin{align}
    H^*_\gamma & = \arg \max_{h \in \{1,\dots,H_{\rm{max}}\}} g(h,T^{[h]}_\gamma), \label{eq:algorithm:opt H T 1}\\
    T^*_\gamma & = T^{[H^*_\gamma]}_\gamma.
    \label{eq:algorithm:opt H T 2}
\end{align}
With $\gamma$ approaching $0$, $(H^*_\gamma,T^*_\gamma)$ approximates the global optimal solution $(H^*,T^*)$ of \text{$\mathcal{P}$1}.

We can find that the number of search steps of {Algorithm \ref{alg:optimization algorithm}} is approximately proportional to $H_{\rm{max}}\log\frac{T_0}{\gamma}$, where $H_{\rm{max}}$ is upper bounded by $\frac{T_0}{\alpha}$ and $T_0 = \frac{L}{v}$. So the computational complexity of {Algorithm \ref{alg:optimization algorithm}} is $\mathcal{O}\left( \frac{L}{\alpha v} \log\frac{L}{\gamma v} \right)$.

\section{Experiments}{\label{sec:Simulation Results}}

In this section, simulation results show that $g(H,T)$ in \eqref{eq:objective function} is a good proxy for the convergence speed of FL-ICV. Two FL tasks are considered: the \emph{beam selection task} for millimeter-wave V2X communication, and the \emph{trajectory prediction task} for autonomous driving. We emphasize that \emph{both tasks are highly relevant for the IoV scenario}.

For the beam selection task, we utilize the Lidar2D NN model proposed in \cite{mashhadi2021} and the Raymobtime s008 dataset (https://www.lasse.ufpa.br/raymobtime/) with 9234 training samples and 1960 validation samples.
For the trajectory prediction task, we utilize the LaneGCN NN model proposed in \cite{lanegcn} and the Argoverse Motion Forecasting dataset (https://www.argoverse.org) with 205942 training samples and 39472 validation samples. 

We train the NN models for these two tasks using the FL-ICV framework. We set the learning rate as $\eta=0.1$, the batch size as $B=64$, the length of the training section as $L=400 \rm{m}$, the arrival rate of ICVs as $\lambda=0.1 \rm{s^{-1}}$, and their velocity as $v =20 \rm{m/s}$. Besides, the communication delays $\tau^{\rm{down}}$ and $\tau^{\rm{up}}$ are both set to $1\rm{s}$, and the probability distribution parameters of computing delay are set to $\alpha=\beta=0.2\rm{s}$. Each ICV carries a local dataset containing $1024$ training samples randomly selected from the total training dataset.

\begin{figure}[htbp]
\centering
    \subfloat[$-\mathcal{L}_{\rm{min}}(T_{\rm{A}})$]
    {\includegraphics[width=0.9\columnwidth,trim= 0 50 0 130,clip]{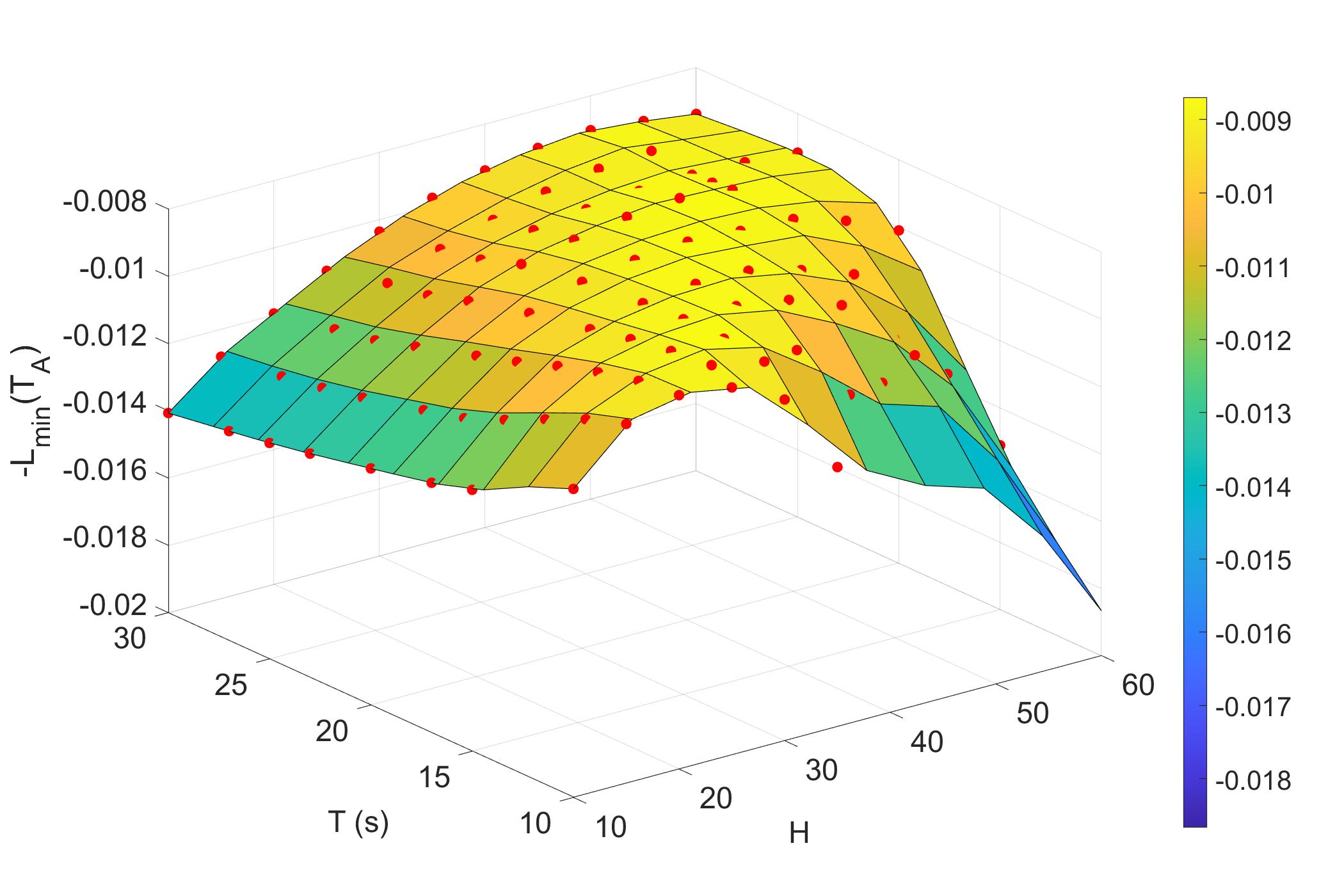}\label{fig:3d compare:beam_loss}}
    \\
    \subfloat[$g(H,T)$]
    {\includegraphics[width=0.9\columnwidth,trim= 0 50 0 130,clip]{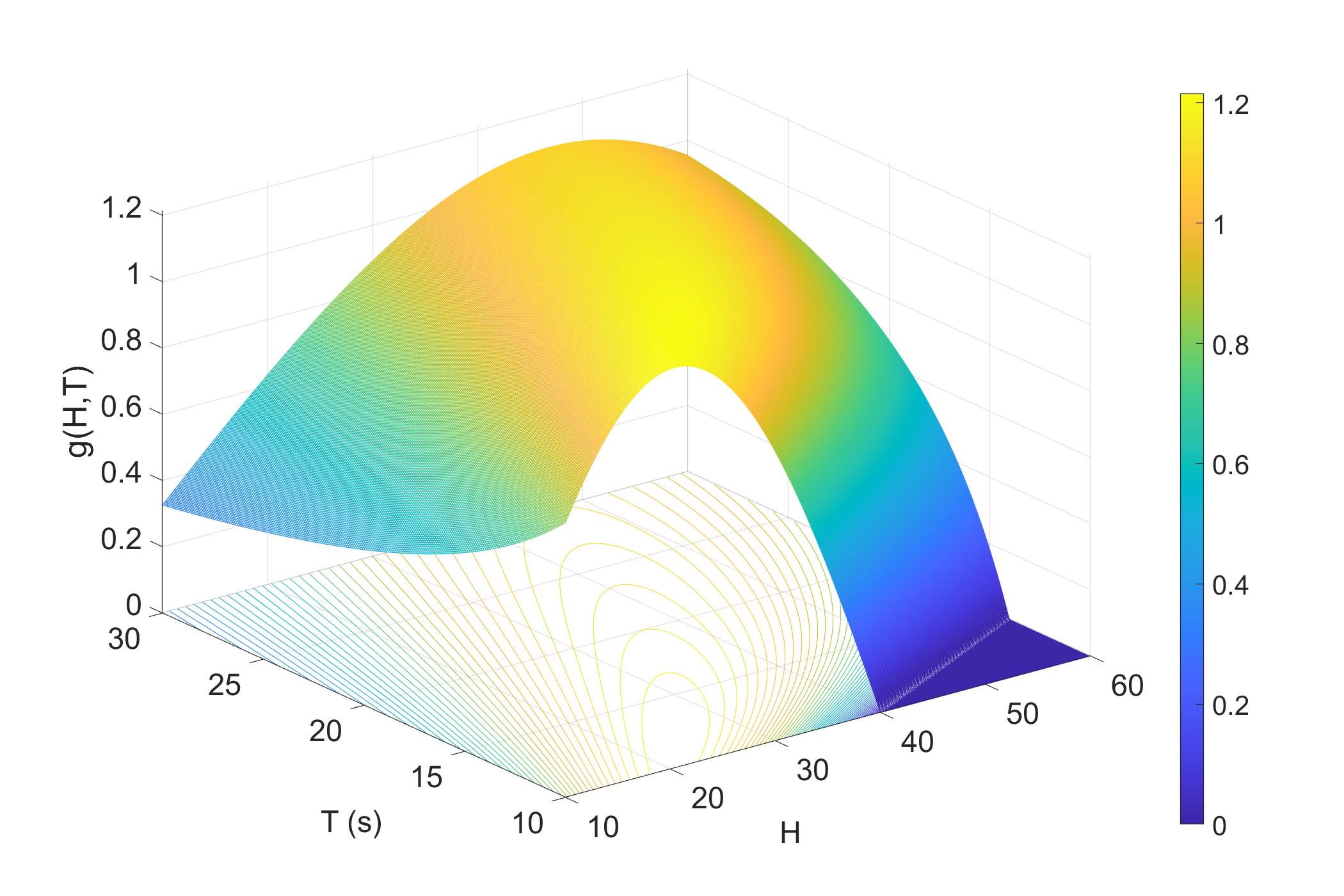}\label{fig:3d compare:beam_g}}
\caption{Comparison between $-\mathcal{L}_{\rm{min}}(T_{\rm{A}})$ and $g(H,T)$ for the beam selection task.}
\label{fig:3d compare}
\end{figure}

\begin{figure}[t]
\centering
    \subfloat[The minADE for the trajectory prediction task.]    {\includegraphics[width=0.9\columnwidth,trim= 0 5 0 5,clip]{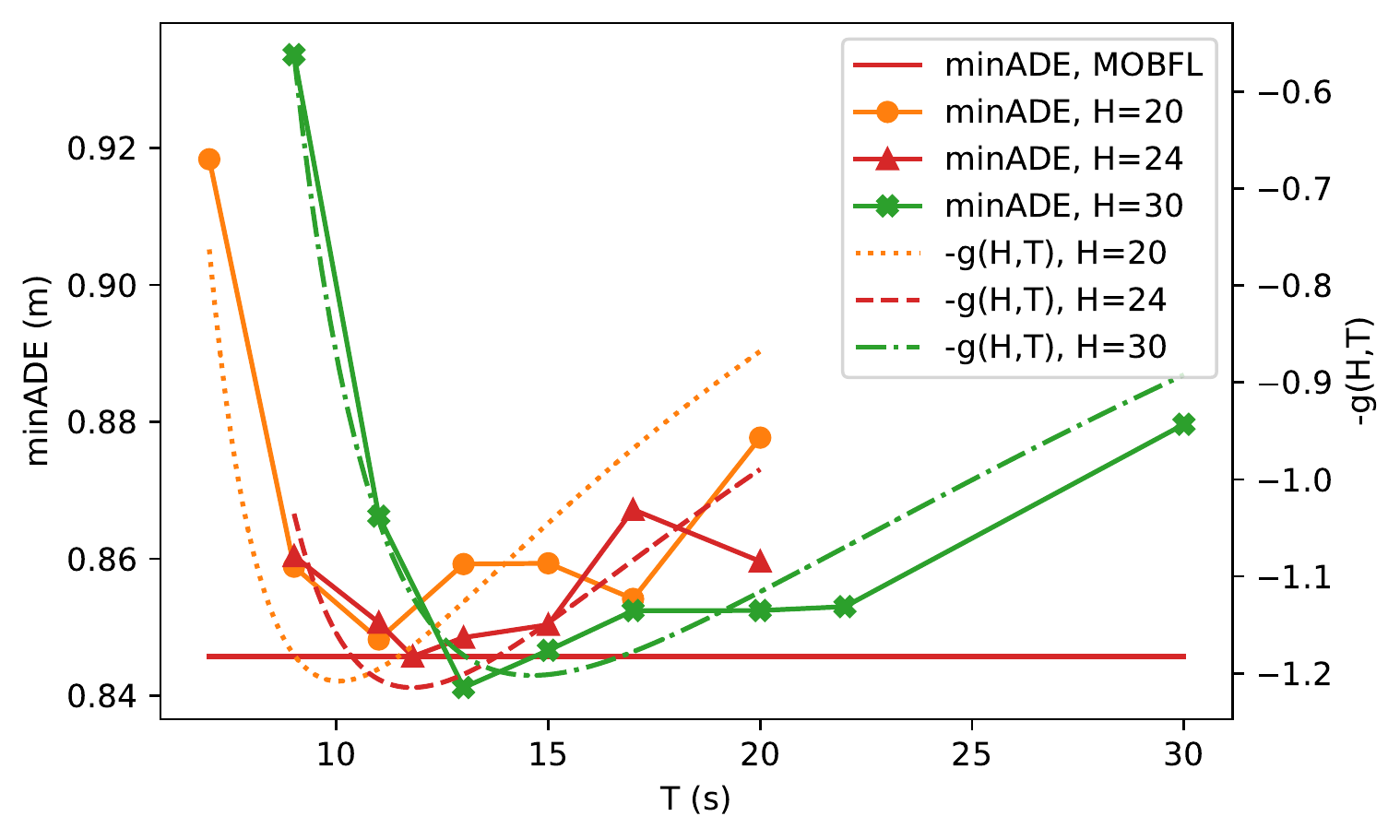}\label{fig:2d compare:traj_minADE}}    \\
    \subfloat[The top-10 accuracy for the beam selection task.]    {\includegraphics[width=0.9\columnwidth,trim= 0 5 0 5,clip]{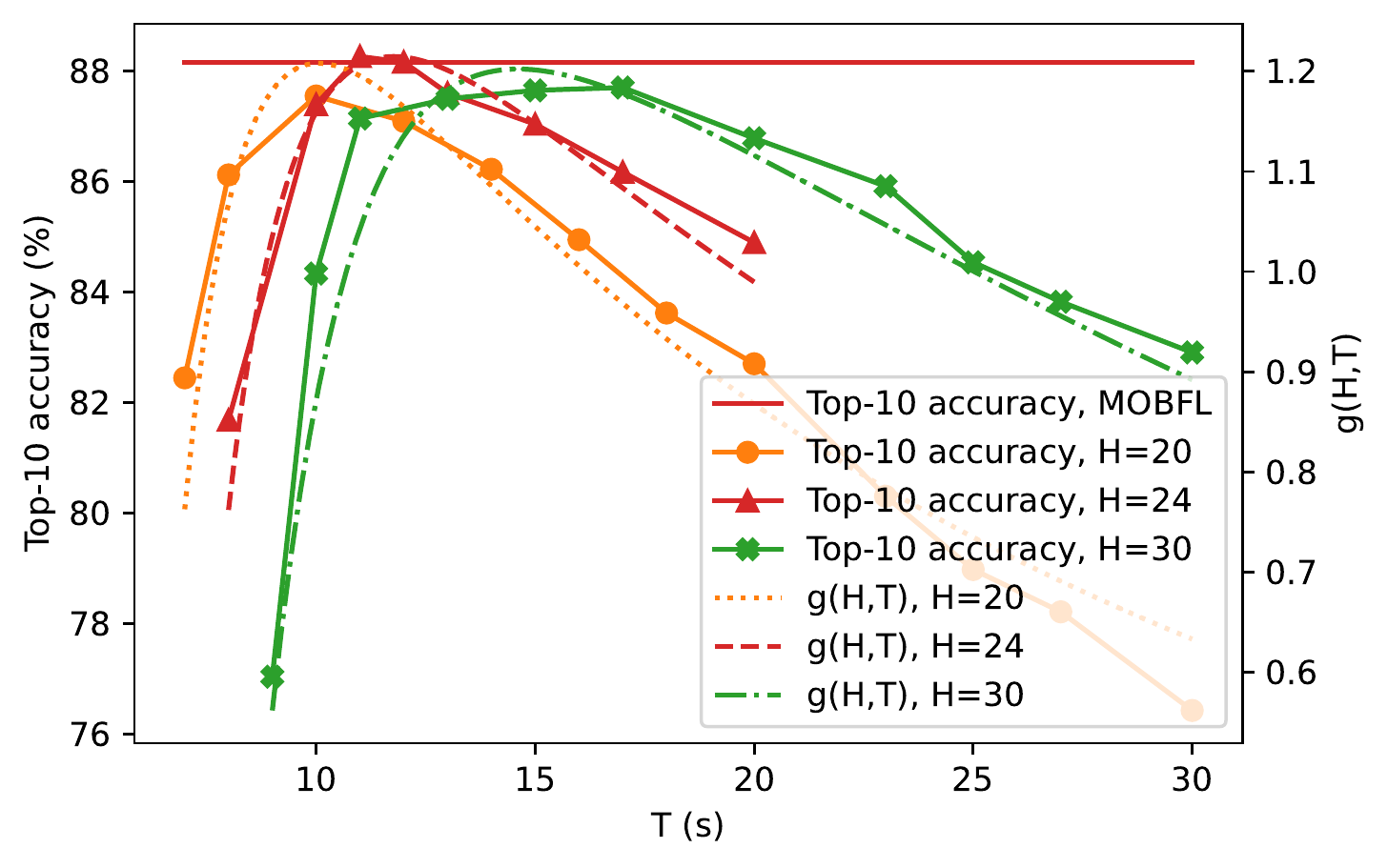}\label{fig:2d compare:beam_top10acc}}
\caption{Performance of the MOB-FL algorithm for both of the beam selection and the trajectory prediction tasks.}
\label{fig:2d compare}
\end{figure}

In Fig. \ref{fig:3d compare}, we compare the $-\mathcal{L}_{\rm{min}}(T_{\rm{A}})$ and $g(H,T)$ in the 3D coordinate system for the beam selection task, with $T_{\rm{A}} = 10000\rm{s}$. Each red point in Fig. \ref{fig:3d compare:beam_loss} represents a different realization of the experiment, with a fitting surface passing through it. The heights of the areas on the surface increase with their color varying from blue to yellow. We can find the convergence performance indicator $-\mathcal{L}_{\rm{min}}(T_{\rm{A}})$ is positively correlated with $g(H,T)$ in Fig. \ref{fig:3d compare:beam_g}, which shows that maximizing $g(H,T)$ is a meaningful proxy to optimize the convergence speed of FL-ICV. Similar results are obtained for the trajectory prediction task, which is omitted due to space limitations.

\addtolength{\topmargin}{0.2in}

To show the performance of the proposed MOB-FL algorithm, we compare the task-specific key performance indicators (KPI) of the corresponding NN models after $T_{\rm{A}}=10000\rm{s}$ time of training for both beam selection and trajectory prediction tasks in Fig. \ref{fig:2d compare}. The minADE in Fig. \ref{fig:2d compare:traj_minADE} corresponds to the minimum of the average displacement errors between the predicted trajectories and the ground truth. The top-10 accuracy in Fig. \ref{fig:2d compare:beam_top10acc} refers to the probability that the optimal beam pair is within the 10 candidates output by the NN model. Each marked point represents a different realization of the experiment. The red horizontal lines represent the MOB-FL-optimized experiment results, with the optimized pair $(H^*,T^*)=(24,11.8\rm{s})$. Each curve, with different line styles, represents the $g(H,T)$ w.r.t. $T$ for a different $H$ value. We observe not only a strong correlation between the KPIs and $g(H,T)$, but also the approximate optimality of MOB-FL algorithm in Fig. \ref{fig:2d compare}.

\section{Conclusion}
In this work, we have proposed the FL-ICV framework, where the set of schedulable ICVs is time-varying due to the mobility of ICVs passing by the BS that orchestrates the FL process. To improve the convergence speed of FL, we have formulated an optimization function $g(H,T)$, which represents the frequency of valid model updates. To maximize $g(H,T)$, we have proposed a mobility-aware optimization algorithm called MOB-FL, which considers the driving speed and the arrival rate of ICVs. Through simulations based on the beam selection and trajectory prediction tasks, we have shown that $g(H,T)$ is a good proxy for the convergence speed of FL-ICV. By optimizing $g(H,T)$ following the MOB-FL algorithm, we achieve an approximate optimal convergence performance.

\begin{appendices}
\section{Proof of Theorem \ref{theorem:1}}{\label{appendix:theorem 1}}
Taking the $k$-th round as example, $\forall m \in \mathcal{M}_k$, its arrival time $\zeta^A_{m}$ belongs to the time interval $ \left(kT-T_0, \left(k+1\right)T\right)$.

When $T \geq T_0$, the time interval can be divided into three sub-intervals, $\left(kT-T_0,kT\right)$, $\left[kT, \left(k+1\right)T-T_0\right)$ and $\left[\left(k+1\right)T-T_0, \left(k+1\right)T\right)$, denoted by $\bar{[1]}$, $\bar{[2]}$ and $\bar{[3]}$, respectively. We define $p^{[k]}_i = \mathbb{P}\left\{ m \in \mathcal{M}^{\rm{suc}}_k | \zeta^A_{m} \in \bar{[i]} \right\}$, which can be transformed to \eqref{eq:proof:p}.
Calculating \eqref{eq:proof:p} based on \eqref{eq:local train delay} and \eqref{eq:Phi1}-\eqref{eq:necessary condition for suc}, we can obtain $p^{[k]}_1$, $p^{[k]}_2$ and $p^{[k]}_3$ in \eqref{eq:proof:p1p3}-\eqref{eq:proof:p2},
\begin{align}
    \label{eq:proof:p} p^{[k]}_i &= \int_{\zeta \in \bar{[i]} } \mathbb{P}\left\{ \Phi^{[1]}_{k,m} < \Phi^{[2]}_{k,m} | \zeta^A_m = \zeta \right\} \cdot \frac{1}{\sigma\left( \bar{[i]} \right)} \cdot d\zeta, \\
    \label{eq:proof:p1p3} p^{[k]}_1  &= p^{[k]}_3  = \frac{1}{T_0}\left[\Xi(H,T) - \beta H \left( 1-e^{-\frac{\Xi(H,T)}{\beta H}} \right) \right], \\
    \label{eq:proof:p2} p^{[k]}_2 &= 1-e^{-\frac{\Xi(H,T)}{\beta H}},
\end{align}
where $\sigma(\bar{[i]})$ is the measure of $\bar{[i]}$.
According to the characteristics of Poisson processes and the independence of successful model uploading of each ICV, $M^{\rm{suc}}_k$ also follows a Poisson distribution with the expectation $\mathbb{E}\left[ M^{\rm{suc}}_k \right] = \sum_{i=1}^3 \lambda \sigma(\bar{[i]}) p^{[k]}_i = \lambda T_0 (p^{[k]}_1 + p^{[k]}_3) + \lambda(T-T_0)p^{[k]}_2 = 2 \lambda\Xi(H,T) + \lambda\left(1-e^{-\frac{\Xi(H,T)}{\beta H}}\right)\left(\left|T-T_0\right|-2 \beta H\right)$.

When $T < T_0$, the time interval can also be divided into three sub-intervals, $\left(kT-T_0,\left(k+1\right)T-T_0\right)$, $\left[\left(k+1\right)T-T_0, kT\right)$ and $\left[kT, \left(k+1\right)T\right)$. We can obtain the same conclusion as the $T \geq T_0$ case following similar steps. 

Given the above, we obtain {Theorem \ref{theorem:1}}.

\section{Proof of Proposition \ref{proposition:T up bound}}{\label{appendix:proposition:T up bound}}

It is apparent that $0<\mathcal{C}_0(H)<1$ when $\Xi(H,T) > 0$.

Based on \eqref{eq:necessary condition for suc}, \eqref{eq:theorem:Lambda} and \eqref{eq:proposition:T up bound}, when $T> \mathcal{T}_{\rm{max}}(H)$, we have
\begin{align}
    \Lambda(H,T) & = 2 \lambda \left(T_0 - \mathcal{T}_{\rm{min}}(H)\right) + \lambda(T-T_0-2 \beta H) \mathcal{C}_0(H) \nonumber \\
    & = \lambda T \mathcal{C}_0(H) + \lambda \mathcal{C}_1(H).
\end{align}
Taking the partial derivative of both sides of \eqref{eq:objective function} w.r.t. $T$, and multiplying by $\frac{1}{H}$, we obtain
\begin{multline}
    \frac{1}{H}\frac{\partial g(H,T)}{\partial T}  = \frac{e^{-\Lambda(H,T)}}{T}\frac{\partial\Lambda(H,T)}{\partial T} - \frac{1-e^{-\Lambda(H,T)}}{T^2} \\
     = \frac{(\lambda T \mathcal{C}_0(H)+1)e^{-\lambda T \mathcal{C}_0(H)-\lambda \mathcal{C}_1(H)}-1}{T^2}.
     \label{eq:theorem:1/H partial g}
\end{multline}
So, a necessary and sufficient condition for  $\frac{\partial g(H,T)}{\partial T} < 0$ is
\begin{equation}
    \lambda T \mathcal{C}_0(H)+1 < e^{\lambda T \mathcal{C}_0(H)+\lambda \mathcal{C}_1(H)}.
    \label{eq:proposition:partial g<0}
\end{equation}

When $\mathcal{C}_1(H) \geq 0$, \eqref{eq:proposition:partial g<0} holds since $x+1 < e^x, \forall x > 0$.

When $\mathcal{C}_1(H) < 0$, $\forall T > \mathcal{T}_{\rm{max}}(H)$, we have $\lambda T\mathcal{C}_0(H)+\lambda \mathcal{C}_1(H) > - 2\lambda \mathcal{C}_1(H) + \frac{1}{4} \geq \sqrt{-2\lambda \mathcal{C}_1(H)} > 0$, where the second inequality follows from $x+\frac{1}{4} \geq \sqrt{x}, \forall x \geq 0$.
It can be transformed to $\lambda T \mathcal{C}_0(H)+1 < 1 + \left(\lambda T \mathcal{C}_0(H)+\lambda \mathcal{C}_1(H)\right) + \frac{\left(\lambda T \mathcal{C}_0(H)+\lambda \mathcal{C}_1(H)\right)^2}{2} \leq e^{\lambda T \mathcal{C}_0(H)+\lambda \mathcal{C}_1(H)}$, which satisfies \eqref{eq:proposition:partial g<0}.

Given the above, we obtain {Proposition \ref{proposition:T up bound}}.

\section{Proof of Proposition \ref{proposition:unimodel}}{\label{appendix:proposition:unimodel}}
When $T > \mathcal{T}_{\rm{min}}(H)$, we have $\frac{\partial \Lambda(H,T)}{\partial T} > 0$ and $\frac{\partial^2 \Lambda(H,T)}{\partial^2 T} \leq 0$, where the equality holds if and only if $T \geq T_0$.

We define $q(H,T) \triangleq \frac{T^2 e^{\Lambda(H,T)}}{H} \frac{\partial g(H,T)}{\partial T}$. Then we have $q(H,T) = T\frac{\partial\Lambda(H,T)}{\partial T} - e^{\Lambda(H,T)} + 1$, and $\forall \,T > \mathcal{T}_{\rm{min}}(H)$,  $\frac{\partial q(H,T)}{\partial T} = \left(1 - e^{\Lambda(H,T)}\right) \frac{\partial\Lambda(H,T)}{\partial T} + T\frac{\partial^2 \Lambda(H,T)}{\partial^2 T} < 0$.

According to the following conditions:
\begin{align}
    & q\left(H,\mathcal{T}_{\rm{min}}(H)\right) = \mathcal{T}_{\rm{min}}(H)\cdot\frac{\partial\Lambda(H,\mathcal{T}_{\rm{min}}(H))}{\partial T} > 0, \nonumber
    \\
    & \frac{\partial q(H,T)}{\partial T} < 0, \forall T \in \left(\mathcal{T}_{\rm{min}}(H), \infty \right), \nonumber
    \\
    & \lim_{T\rightarrow\infty} \frac{\partial q(H,T)}{\partial T} = 
    \lim_{T\rightarrow\infty} \left(1 - e^{\Lambda(H,T)}\right) \frac{\partial \Lambda(H,T)}{\partial T} = - \infty, \nonumber
\end{align}
we can prove that $q(H,T)$ for a fixed $H \in \mathbb{N}^+$ has exactly one zero point for $T \in \left(\mathcal{T}_{\rm{min}}(H), \infty \right)$. Since $q(H,T)$ and $\frac{\partial g(H,T)}{\partial T}$ have the same sign, {Proposition \ref{proposition:unimodel}} is proved.

\end{appendices}
\bibliographystyle{ieeetr}
\bibliography{reference.bib}
\end{document}